\documentclass[letterpaper]{article} 
\usepackage{aaai24}  
\usepackage{times}  
\usepackage{helvet}  
\usepackage{courier}  
\usepackage[hyphens]{url}  
\usepackage{graphicx} 
\usepackage{natbib}  
\usepackage{caption} 
\usepackage{newfloat}
\usepackage{listings}
\DeclareCaptionStyle{ruled}{labelfont=normalfont,labelsep=colon,strut=off} 
\usepackage[ruled,vlined, linesnumbered ]{algorithm2e}
\SetCommentSty{emph}
\usepackage{algpseudocode}
\usepackage{graphicx}
\usepackage[]{svg}

\usepackage{amsthm}
\usepackage{caption}
\usepackage{subcaption}
\usepackage{mathtools}
\usepackage{amssymb}

\theoremstyle{definition}
\newtheorem{definition}{Definition}
\newtheorem{theorem}{Theorem}
\newtheorem{property}{Property}
\newtheorem{corollary}{Corollary}
\newtheorem{lemma}{Lemma}

\usepackage{cleveref}
\title{Bidirectional Temporal Plan Graph: Enabling Switchable Passing Orders for More Efficient Multi-Agent Path Finding Plan Execution}
\author{
    Yifan Su, Rishi Veerapaneni, Jiaoyang Li
}
\affiliations{
    Carnegie Mellon University\\

    \{yifansu,rveerapa\}@andrew.cmu.edu, jiaoyangli@cmu.edu
}

\usepackage{bibentry}

\begin{document}

\maketitle

\begin{abstract}
The Multi-Agent Path Finding (MAPF) problem involves planning collision-free paths for multiple agents in a shared environment. The majority of MAPF solvers rely on the assumption that an agent can arrive at a specific location at a specific timestep. However, real-world execution uncertainties can cause agents to deviate from this assumption, leading to collisions and deadlocks. Prior research solves this problem by having agents follow a Temporal Plan Graph (TPG), enforcing a consistent passing order at every location as defined in the MAPF plan. However, we show that TPGs are overly strict because, in some circumstances, satisfying the passing order requires agents to wait unnecessarily, leading to longer execution time. To overcome this issue, we introduce a new graphical representation called a Bidirectional Temporal Plan Graph (BTPG), which allows switching passing orders during execution to avoid unnecessary waiting time. We design two anytime algorithms for constructing a BTPG: BTPG-naïve and BTPG-optimized. Experimental results show that following BTPGs consistently outperforms following TPGs, reducing unnecessary waits by 8-20\%.
\end{abstract}

\section{Introduction}
The Multi-Agent Path Finding (MAPF) problem involves finding paths for multiple agents to reach their respective destinations from their starting points in a shared environment without collisions \cite{stern_multi-agent_2021}.
The importance of solving the MAPF problem is reflected in its wide range of applications, such as warehouse automation \cite{varambally_which_2022}, traffic management \cite{li_scalable_2021}, and drone swarm coordination \cite{honig_trajectory_2018}.

The MAPF problem discretizes time into unit timesteps, assuming that an agent can reach a particular location at a particular timestep. However, in practice, agents face challenges such as communication delays, physical constraints, or hardware failures that prevent them from meeting this assumption. When an agent is unable to reach a location required by the MAPF plan at a specific timestep, it may result in deadlocks or collisions with other agents.


To address this issue, \citet{ma_multi-agent_2017} proposed a MAPF with Delay Probability (MAPF-DP) model where, at each timestep, each agent has a certain probability of stopping in place, leading to delays. Both \citet{ma_multi-agent_2017} and \citet{hoenig_multi-agent_2016} show that as long as the agents adhere to the inter-agent dependencies required by the MAPF plan, they can complete their respective paths without collisions or deadlocks. These dependencies ensure that the order in which the agents pass through \emph{every} location is consistent with the planned path. Both studies capture these dependencies by defining a graph representation, called the \emph{Temporal Plan Graph} (TPG). Requiring agents to execute according to the dependencies in the TPG ensures completion of the MAPF plan without collisions and deadlocks, even if agents reach some locations at different timesteps than specified in the MAPF plan.

\begin{figure}[t]
    \centering
    \begin{subfigure}[b]{0.12\textwidth}
         \centering
         \includegraphics[width=\textwidth]{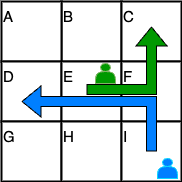}
         \subcaption{Scenario 1.}
         \label{motivating1}
    \end{subfigure}\hfill
    \begin{subfigure}[b]{0.12\textwidth}
         \centering
         \includegraphics[width=\textwidth]{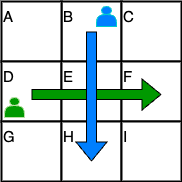}
         
         \subcaption{Scenario 2.}
         \label{motivating2}
    \end{subfigure}\hfill
    \begin{subfigure}[b]{0.2\textwidth}
         \centering
         \includegraphics[width=\textwidth]{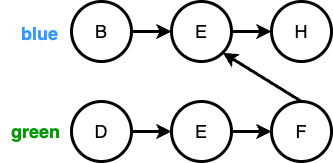}
             \subcaption{TPG for Scenario 2.}
             \label{fig: adapted TPG}
    \end{subfigure}
    \caption{Motivating examples. In Scenario 1 (2), the MAPF plan requries the agents to follow the arrows with the green agent passing through location F (E) before the blue agent.}
    \label{motivating}

\end{figure}



In Figure \ref{motivating1}, the MAPF plan specifies that the green agent should pass location F before the blue agent. If they follow the corresponding TPG, the green agent will pass F first. Switching this order would lead to a deadlock, with the blue agent trying to move from F to E and the green agent from E to F.
In contrast, Figure \ref{motivating2} shows a different scenario. Here, the green agent is supposed to pass location E first. If the green agent is delayed at D, the TPG makes the blue agent wait at B until the green agent passes E. However, if the green agent is delayed, the blue agent can safely pass E first without causing a deadlock. This flexibility in switching the order of passing can reduce unnecessary waiting and shorten the overall execution time.


The goal of this paper is to determine when the dependencies can be switched, allowing agents to pass such locations in a ``first-come-first-served'' manner in execution to minimize unnecessary waiting time.  Our main contributions are:
\begin{enumerate}
    \item Defining a new graphical representation called the \emph{Bidirectional Temporal Planning Graph} (BTPG) for capturing all such switchable dependencies in a MAPF plan;
    \item Introducing sufficient conditions for a BTPG to be provably collision-free and deadlock-free;
    \item Proposing two anytime algorithms for constructing a BTPG, namely
    BTPG-naïve and BTPG-optimized, and showing that following BTPGs consistently outperforms following TPGs, reducing unnecessary waits by 8-20\%.
\end{enumerate}


\section{Preliminaries}
\subsection{Problem Definition}
The input of the MAPF problem \cite{stern_multi-agent_2021} contains an undirected graph and a set of $n$ agents, each with a start location and a target location. Its output is a MAPF plan, which is a set of conflict-free paths for $n$ agents that move them from their respective start locations to their respective target locations. Time is discretized into unit timesteps. At every timestep, an agent either moves to a neighboring location or waits at its current location. The $i$th element in an agent's path represents the location of that agent at timestep $i$. An agent that finishes its path rests at its target location forever. There are two ways in which the paths of two agents can create a conflict: A \emph{vertex conflict} occurs when two agents occupy the same location at the same timestep, and a \emph{edge conflict} occurs when two agents traverse the same edge at the same timestep.

To enable the ``first-come-first-served'' mechanism, we assume that, during execution, every agent has an onboard collision-avoidance mechanism that can stop automatically when there is an obstacle (such as another agent) in front of it, a common function for modern robots.

In order to reflect the execution uncertainty in real applications, we use the MAPF with Delay Probabilities (MAPF-DP) model \cite{ma_multi-agent_2017}. That is, during execution, an agent may get stuck at its current location for a period of time instead of following its path to move to its next location. In addition to the two types of conflicts defined above, the original MAPF-DP model also disallows \emph{following conflicts}, where one agent enters a vertex that was occupied by another agent in the previous timestep. Their concern is that if the front agent suddenly stops, the following agent may run into the front agent. However, we allow the following conflicts in this paper for three reasons. First, the onboard collision-avoidance mechanism can easily guarantee collision-freeness in the above-mentioned situation because the following agent will stop when the front agent stops. Second, allowing following conflicts makes MAPF plans shorter and thus potentially leads to better execution times. Third, most modern MAPF planners, such as CBS-based~\cite{sharon_conflict-based_2015,gange_lazy_2021,li_eecbs_2021}, priority-based \cite{erdmann_multiple_1986,ma_searching_2019}, and rule-based \cite{sajid_multi-agent_2021,okumura_priority_nodate} planners allow for following conflicts, so we can directly use these off-the-shelf MAPF planners without changes.  

\subsection{Temporal Plan Graph (TPG)}

We now introduce the definition of TPGs along with its properties established in prior research. In line with previous work, we assume that following conflicts are disallowed here. In the next section, we will demonstrate how we modify the definition and properties of TPGs to accommodate scenarios where such following conflicts are indeed allowed.

\begin{definition}[TPG]
\label{def: TPG}
TPG \cite{ma_multi-agent_2017} is a directed acyclic graph $G=(V,E)$. 
Each vertex $v_i^m \in V$ represents a state of agent $m$ being at location $loc(v_i^m)$. Here, $i$ indicates that $loc(v_i^m)$ is the $i$th element in the path of agent $m$. 
Each edge establishes a precedence dependency between two states. We divide the edges $E$ into  type-1 edges $E_1$ and type-2 edges $E_2$. A type-1 edge $(v_{i}^m, v_{i+1}^m) \in E_1$ forces agent $m$ to enter $loc(v_{i+1}^m)$ only after it has entered $loc(v_{i}^m)$; A type-2 edge $(v_{i}^m, v_{j}^n) \in E_2$ forces agent $n$ to enter $loc(v_{j}^n)$ only after agent $m$ has entered $loc(v_{i}^m)$.
\end{definition}

\paragraph{TPG Construction}
To construct a TPG from a given MAPF plan, for every agent, we introduce a vertex for every element in its path that is different from the previous one, i.e., we omit states indicating waits. We add a type-1 edge for each pair of successive elements in the path of every agent. We add a type-2 edge if one agent visits a location before another. More specifically, if there is a \emph{conflict location} $loc(v_{i}^m)=loc(v_{j}^n)$, and agent $m$ visits it before agent $n$ in the MAPF plan, we introduce a type-2 edge $(v_{i+1}^m, v_{j}^n)$, indicating that only after agent $m$ left $loc(v_{i}^m)$ and enters its next location $loc(v_{i+1}^m)$, agent $n$ is allowed to enter $loc(v_{j}^n)$. 
Figure \ref{fig: adapted TPG} illustrates the TPG derived from the MAPF plan depicted in Figure \ref{motivating2}. For simplicity, we use $loc(v_i^m)$ to denote a vertex $v_i^m$. The arrow FE is a type-2 edge indicating that the green agent needs to pass through location E before the blue agent. Other arrows represent type-1 edges.

\paragraph{TPG Execution Policy}
At each timestep during execution, each agent $m$ is allowed to move to its next state $v_i^m$ only if, $\forall (v_j^n, v_i^m) \in E_2$, agent $n$ has already visited state $v_j^n$; otherwise, it must wait~\cite{ma_multi-agent_2017}.

\begin{property}[Cycle-free $\Leftrightarrow$ valid TPG]
\label{prop: valid TPG 1}
If and only if a TPG is cycle-free, TPG execution is guaranteed to succeed, i.e., agents are assured to reach their target locations without collisions or deadlocks~\cite{berndt_feedback_2020, coskun_deadlock-free_2021}. We refer to such a TPG as a \emph{valid} TPG.
\end{property}

\begin{figure}[t]
        \begin{minipage}[b]{0.23\textwidth}
             \centering
             \includegraphics[width=.51\textwidth]{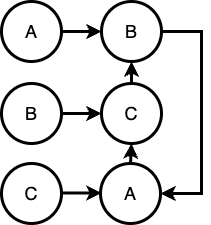}
             \caption{TPG for the scenario where 3 agents rotate in a cycle simultaneously.}
             \label{fig: rotation}
        \end{minipage}\hfill
         \begin{minipage}[b]{0.22\textwidth}
             \centering
            \includegraphics[width=.85\textwidth]{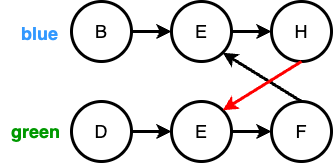}
             \caption{BTPG for the example depicted in \Cref{motivating2}.}
             \label{fig: new BTPG}
         \end{minipage}
\end{figure}

\subsection{Related Work}

The TPG execution policy asks agents to coordinate their passing orders at conflict locations, which sometimes leads to unnecessary waits. 
One way to reduce such unnecessary waits is to find MAPF plans that explicitly minimize coordination, such as SL-CBS~\cite{wagner_minimizing_2022} and DBS~\cite{okumura_offline_2023}. However, these models usually impose overly strict coordination constraints, making them more challenging to solve than regular MAPF and resulting in longer execution times.

Another way is to generate MAPF plans that consider potential delays during planning~\cite{ma_multi-agent_2017, atzmon_robust_2021}, which, however, demand prior knowledge of delays and are sometimes overly conservative.

A third way is to replan, for example, by checking if a particular passing order can be switched when delays happen during execution~\cite{berndt_feedback_2020,paul_fast_2023,pecora_loosely-coupled_2018,coskun_deadlock-free_2021}. However, these methods require additional computation during execution, potentially introducing further delays.

In contrast, our proposed BTPG idea generates a MAPF plan using regular MAPF planners and post-processing it by exploring all potential switchable dependencies during the planning phase. By doing so, we avoid increasing the complexity of solving MAPF, eliminate the need for prior knowledge of delay distribution, and eradicate the necessity for additional computation during execution.





\section{TPG that Allows Following}
In the TPG shown in \Cref{fig: adapted TPG}, suppose at timestep $t$, the blue and green agents are at B and E, respectively. However, at timestep $t+1$, the blue agent cannot reach E even if the green agent departs E at the same timestep. This is because the original definition of type-2 edges was predicated on disallowing following conflicts, which requires the blue agent to reach E only \emph{after} the green agent reaches F.

However, since we allow such following actions, we propose a refined definition for type-2 edges $(v_i^m, v_j^n)$. The revised definition specifies that agent $n$ can enter $loc(v_{j}^n)$ \emph{no earlier than} agent $m$ enters $loc(v_{i}^m)$.
With this adjustment, in \Cref{fig: adapted TPG}, the blue agent can enter E simultaneously with the green agent departing E for F.

It is worth noting that this change is orthogonal to our proposed BTPG techniques. Our techniques are applicable to both TPGs that allow following and those that do not.

 \subsection{Rotation Cycle}
Since agents can follow each other, they can rotate simultaneously, which was not allowed before in the original TPG definition. \Cref{fig: rotation} shows an example, where, at the same timestep, agent 1 goes from A to B, agent 2 goes from B to C, and agent 3 goes from C to A.
This leads to a cycle ($B\rightarrow A \rightarrow C \rightarrow B$) in the TPG but not a deadlock. We thus need to update the definition of valid TPGs.
 \begin{definition}[Rotation Cycle]
\label{def: rotation cycle}
   A \emph{rotation cycle} is a cycle in a TPG consisting of only type-2 edges, with the cycle containing more than two edges. Note that if a cycle contains only two edges, it leads to a deadlock, since rotating the corresponding two agents is an edge conflict. 
\end{definition}
\begin{property}[Valid TPG with rotations]
\label{prop: valid TPG }
A TPG (that allows following) is valid if it has no non-rotation cycles.
\end{property}

\section{Bidirectional TPG}\label{sec:BTPG}

\Cref{fig: adapted TPG} depicts the TPG for Scenario 2 (\Cref{motivating2}) where the green agent enters E before the blue agent. If we switch the passing order, letting the blue agent enter E first, then we must change the type-2 edge from $(F, E)$ to $(H, E)$. We refer to these two type-2 edges as a bidirectional pair.
\begin{definition}[Bidirectional pair]
\label{def: bType2}
A \emph{bidirectional pair} $(e, \tilde{e})$ consists of two type-2 edges $e = (v_{i+1}^m, v_{j}^n)$ and $\tilde{e} = (v_{j+1}^n, v_{i}^m)$ with conflict location $loc(v_{j}^n) = loc(v_{i}^m)$. We refer to one such edge as the \emph{reversed edge} of the other. 
\end{definition}
\begin{definition}[BTPG]
\label{def:BTPG}
    A \emph{Bidirectional TPG} (BTPG) is a TPG that contains bidirectional pairs. We use $E_{pair}$ to represent the set of type-2 edges that are in bidirectional pairs.
\end{definition}

\paragraph{BTPG Execution Policy}
The BTPG execution policy follows the TPG execution policy except that, for each bidirectional pair, only one edge is selected during execution based on a ``first-come-first-served'' manner. Specifically, either agent in the pair is allowed to enter the conflict location first. But when the first agent arrives, the type-2 edge that enables this agent to enter first is selected, while the other edge is discarded. In other words, edge $(v_{i+1}^m, v_j^n)$ is selected when agent $m$ reaches $v_i^m$, and edge $(v_{j+1}^n, v_{i}^m)$ is selected when agent $n$ reaches $v_j^n$.
\Cref{fig: new BTPG} displays the BTPG for Scenario 2.
Edges $(F,E)$ and $(H,E)$ are a bidirectional pair, indicating that both agents can pass through location E first. If, for example, the green agent reaches E first, then $(F,E)$ is selected, and $(H,E)$ is discarded.

\begin{figure*}[h]
    \centering
    \includegraphics[width=\textwidth]{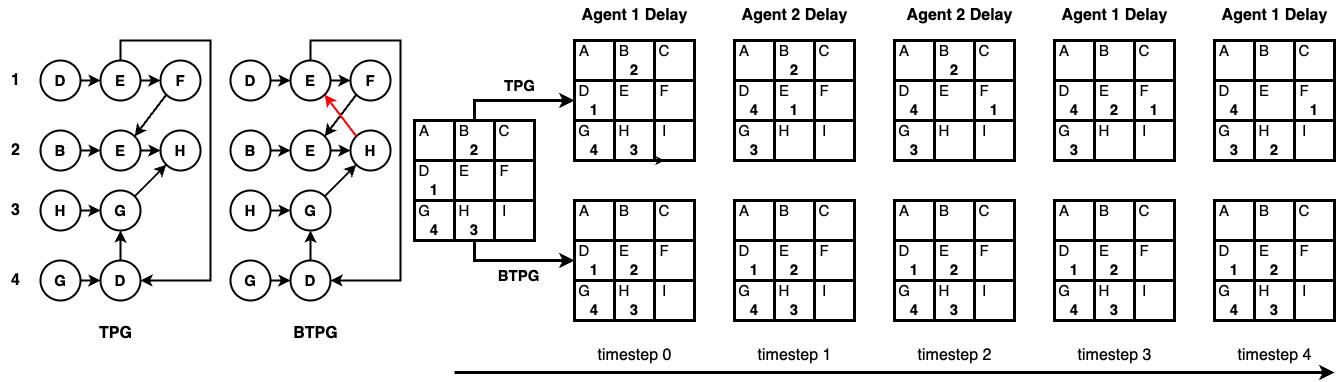}
    \caption{In our 3,900 simulations, BTPG never performed worse than TPG. 
    Thus ideally we could prove that BTPGs are strictly superior to TPGs. 
    However, here we show a hand-crafted adversarial example where a BTPG leads to longer execution time than a TPG under a specific set of delays.
    }
    \label{fig:counter-ex}
\end{figure*}
\paragraph{Execution Time of TPG vs BTPG} While we introduce BTPG to reduce unnecessary waits caused by certain type-2 edges in TPG, the execution time of following the BTPG policy can be longer than that of following the TPG policy in adversarial cases. We provide an example in \Cref{fig:counter-ex}, even though we have not observed any such cases among the 3,900 simulations tested in our experiments. 


\section{Construct BTPG from TPG}
To construct a BTPG, we first run a MAPF planner to obtain a MAPF plan and convert it to a valid TPG. We then check one by one whether each type-2 edge can be transformed into a bidirectional pair without resulting in an invalid BTPG. We terminate the algorithm when we reach the runtime limit or when all type-2 edges have been evaluated. We develop two such algorithms: BTPG-n and BTPG-o.  

\subsection{BTPG-naïve (BTPG-n)}
Since exactly one edge within each bidirectional pair will be selected during execution, a BTPG with $k$ bidirectional pairs can be conceptualized as comprising a collection of $2^k$ TPGs, each of which represents a different combination of edges in bidirectional pairs. Depending on the order in which the agents pass through the conflict location of each bidirectional pair, one of these TPGs will eventually be executed. Therefore, when all $2^k$ TPGs are valid, namely none of them contains any non-rotation cycles (according to Property \ref{prop: valid TPG }), the corresponding BTPG is guaranteed to be valid. Concurrently, not viewing a BTPG as $2^k$ TPGs but viewing it as a single graph with bidirectional pairs, a BTPG is valid if it contains only rotation cycles and self cycles, where a \emph{self cycle} is a cycle that involves both edges in a bidirectional pair, like $E\rightarrow H \rightarrow E \rightarrow F \rightarrow E$ shown in \Cref{fig: new BTPG}. A valid BTPG can have self cycles because a self cycle can never appear in one TPG as a TPG cannot contain both edges in a bidirectional pair. 

\begin{property}[Valid BTPG-naïve]
\label{prop: valid BTPG - naive}
    A BTPG is valid if it does not contain any Non-Rotation and Non-Self (NRNS) cycles.
\end{property}

Based on Property \ref{prop: valid BTPG - naive}, our first proposed algorithm BTPG-naïve (BTPG-n) (see Algorithm \ref{algo:Constructing BTPG - baseline}) examines type-2 edges one by one to detect NRNS cycles and change an edge to bidireticonal pairs if no NRNS cycles are found. As $\mathcal{G}$ is guaranteed to have no NRNS cycles at the beginning of each ``for'' iteration (\Cref{line:e-tilde}), our focus is solely on checking for NRNS cycles involving the newly added edge $\tilde{e}$ at \Cref{line:call-hasCycle}. This verification is carried out by running a Depth First Search (DFS) (see Algorithm \ref{algo: hasCycle}) from $v_{i}^m$ to determine if we can reach $v_{j+1}^n$ through NRNS cycles. 
Note that \Cref{line:for-loop,line:e-tilde} indicate that $v_{i+1}^m$ is not the first state of agent $m$, and $v_{j}^n$ is not the last state of agent $n$. We do not examine type-2 edges pointing from the first state of an agent because it starts at its first state, so no other agents can visit the corresponding conflict location before it. Similarly, we do not examine those pointing to the last state of an agent because it stays there without leaving, so no other agents can visit the corresponding conflict location after it.
\begin{algorithm}[t!]
        \caption{BTPG-naïve/optimized. The boxed \fbox{code} is only for BTPG-optimized.}
        
        \label{algo:Constructing BTPG - baseline}
        \SetKwFor{While}{\fbox{while $E_{pair}$ has been updated do}}{}
        
        \SetKwProg{Fn}{Function}{:}{end}
        \KwIn{TPG $G=(V, E_1 \cup E_2)$} 
        \KwOut{BTPG $\mathcal{G}$}
        \LinesNumbered
        \texttt{$E_{pair} \gets \emptyset$}\tcp*{set of edges in bidirectional pairs}
        \While{\label{line:Epair-updated}}
        {
            \For{\texttt{$e=(v_{i+1}^m, v_{j}^n)$} \textbf{in} \texttt{$E_2$} (or \textbf{until} \texttt{TimeOut})\label{line:for-loop}}
            {
                \texttt{$\tilde{e} \gets (v_{j+1}^n, v_{i}^m)$}\tcp*{reversed edge}\label{line:e-tilde}
                $E_{pair} \gets E_{pair} \cup \{e,\tilde{e}\}$,
                $E_2 \gets E_2 \setminus \{e\}$\;
                $\mathcal{G}\gets(V,E_1 \cup E_2 \cup E_{pair})$\;
                \If{\texttt{hasCycle}($\mathcal{G}, v_{i}^m, v_{j+1}^n, \fbox{$\{v_{i}^m\},$} \{e\}$)\label{line:call-hasCycle}}
                { 
                    $E_{pair} \gets E_{pair} \setminus \{e,\tilde{e}\}$,
                    $E_2 \gets E_2 \cup \{e\}$\;
                 }
            }
        }
        
        \Return $\mathcal{G}=(V,E_1 \cup E_2 \cup E_{pair})$\;
\end{algorithm}

\begin{algorithm}[t!]
        \caption{hasCycle. The boxed \fbox{code} is only for BTPG-optimized.}
        \label{algo: hasCycle}
        \SetKwFunction{MyFunction}{MyFunction}
        \SetKwFor{mIf}{\textcolor{blue}{if}}{\textcolor{blue}{then}}{\textcolor{blue}{end}}
        \SetKwProg{Fn}{Function}{:}{end}
        \KwIn{(1) BTPG $\mathcal{G} = (V, E_1 \cup E_2 \cup E_{pair})$, (2) current vertex for expansion $v_{i}^m$, (3) origin vertex $v_o$, and (4) set(s) of visited \fbox{vertices $V_{vis}$ and} edges $E_{vis}$ along the current DFS branch.} 
        \KwOut{\texttt{true} or \texttt{false}}
        \LinesNumbered
        \If{$v_i^m = v_o$\label{line:cycle-found}}
                {
                \uIf{$E_{vis} \subset E_2 \cup E_{pair}$ \textbf{and} $ |E_{vis}| > 2 $\label{line:detect-rotation-cycle}}{\Return \texttt{false}\tcp*{rotation cycle}\label{line:rotation-cycle}}
                \lElse{\Return \texttt{true}\label{line:non-rotation-cycle}}
                }
            \For{\texttt{$v_j^n$} \textbf{in} $\{v_j^n \in V \mid (v_i^m, v_j^n) \in E_1 \cup E_2 \cup E_{pair}$\}} 
            {
                $e \gets (v_i^m, v_j^n)$\;
                
                \If{\texttt{$e \in E_{pair}$}\label{line:e-in-pair}}
                {
                    \lIf{$(v_{j+1}^n, v_{i-1}^m) \in E_{vis}$ \fbox{\textbf{or} $\exists v_{i'}^m \in V_{vis}: i' < i$}\label{line:skip-e}}
                    {\texttt{continue}\label{line:skip-e2}}
                }
                
                \lIf{\texttt{hasCycle}($\mathcal{G}$, $v_j^n$, $v_o$, \fbox{$V_{vis} \cup \{v_j^n\}$,} $E_{vis} \cup \{e\}$)}
                {\Return \texttt{true}}
            }
            \Return \texttt{false}\;
             
\end{algorithm}

\Cref{algo: hasCycle} returns true if and only if it detects a NRNS cycle. It makes two key changes to regular DFS to exclude rotation and self cycles.  
First, since only one edge within each bidirectional pair can be selected during execution, our DFS must avoid visiting both edges in a bidirectional pair along one DFS branch. To do so, we use $E_{vis}$ to store the edges that our DFS has visited along the current branch and prune a child node if the reversed edge of the current type-2 edge is in $E_{vis}$ (\Cref{line:e-in-pair,line:skip-e}).
Second, when our DFS finds a cycle (\Cref{line:cycle-found}), we return false if it is a rotation cycle (\Cref{line:detect-rotation-cycle,line:rotation-cycle}) and true otherwise (\Cref{line:non-rotation-cycle}). Here, $E_{vis}$ to store the edges that our DFS has visited along the current branch.An additional implementation detail not explicitly outlined in this pseudo-code is that, to speed up our DFS, we maintain a set of visited vertices to prevent the algorithm from re-expanding the same vertex.

\subsubsection{Grouping}
\begin{figure}[t]
         \centering
         \includegraphics[width=0.45\textwidth]{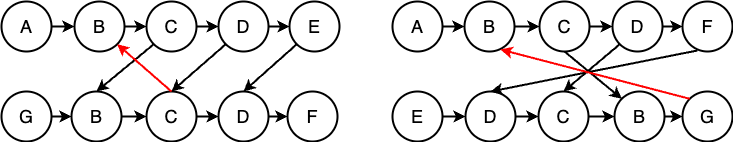}
     
        \caption{Two cases of grouping: The left case is when one agent follows another agent; the right one is when two agents move along the same path in opposite directions. Red arrows depict the reversed edges of edge $(C,B)$.}
        \label{fig:three TPGs}
\end{figure}

Motivated by \cite{berndt_feedback_2020}, we find two common cases in which BTPG-n cannot change the type-2 edges to bidirectional pairs shown in Figure \ref{fig:three TPGs}. They are the cases where two agents visit the same sequence of locations in the same or reversed order, which have consecutive type-2 edges in corresponding TPGs. As adding the reserved edge of any type-2 edge can form a NRNS cycle, BTPG-n cannot change any type-2 edges in such cases to bidirectional pairs.
Therefore, at \Cref{line:for-loop} of Algorithm \ref{algo:Constructing BTPG - baseline}, we merge such type-2 edges into groups and examine only type-2 edges that cannot form a group to reduce the runtime.

\subsection{BTPG-optimized (BTPG-o)}

\begin{figure}[t]
    \begin{subfigure}[a]{0.4\linewidth}
         \centering
         \includegraphics[width=0.87\linewidth]{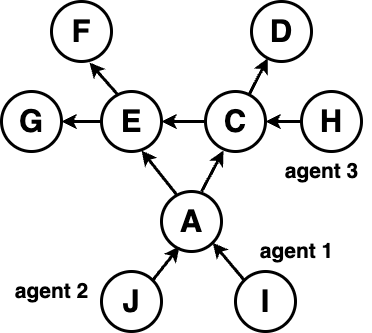}
         \caption{MAPF instance.}
         \label{fig: example path}
     \end{subfigure}
     \hfill
     \begin{subfigure}[a]{0.59\linewidth}
         \centering
         \includegraphics[width=.9\linewidth]{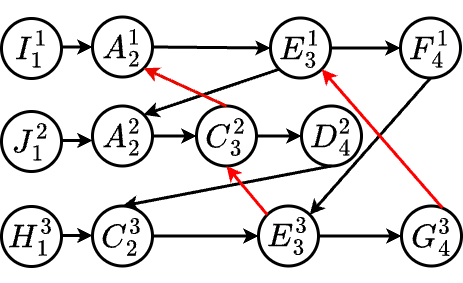}
         \caption{Corresponding BTPG. We use $X_i^m$ to depict vertex $v_i^m$ with $loc(v_i^m)=X$.}
         \label{fig: example BTPG}
     \end{subfigure}

     \caption{Example where some $2^k$ TPGs may never occur.}\label{fig:counter-example}
\end{figure}


Examining all $2^k$ TPGs is overly conservative as we find that not all $2^k$ TPGs can occur in practice. Consider the example in \Cref{fig:counter-example}. If we add bidirectional pairs for all three type-2 edges, the resulting BTPG admits a TPG with type-2 edges $(E^3_3,C_3^2)$, $(C_3^2,A_2^1)$, and $(F^1_4,E^3_3)$. BTPG-n would regard this BTPG as invalid as it contains a NRNS cycle $E^3_3 \rightarrow C^2_3 \rightarrow A^1_2  \rightarrow E^1_3 \rightarrow F^1_4 \rightarrow E^3_3$.

If we analyze this TPG, it leads to a deadlock where agent 1 is at location I, awaiting agent 2 to enter location C, agent 2 is at location A, awaiting agent 3 to enter E, and agent 3 is at location C, awaiting agent 1 to enter F. However, since agent 1 and 3 are at I and C, respectively, according to the BTPG execution policy, edge $(G_4^3,E_3^1)$ rather than edge $(F_4^1,E_3^3)$ would be selected. 
Thus, this TPG would never be encountered during execution. Next, we will go through several lemmas to formally reason about such cases.

\begin{figure*}[th!]
     \centering
     \includegraphics[width=0.95\textwidth]{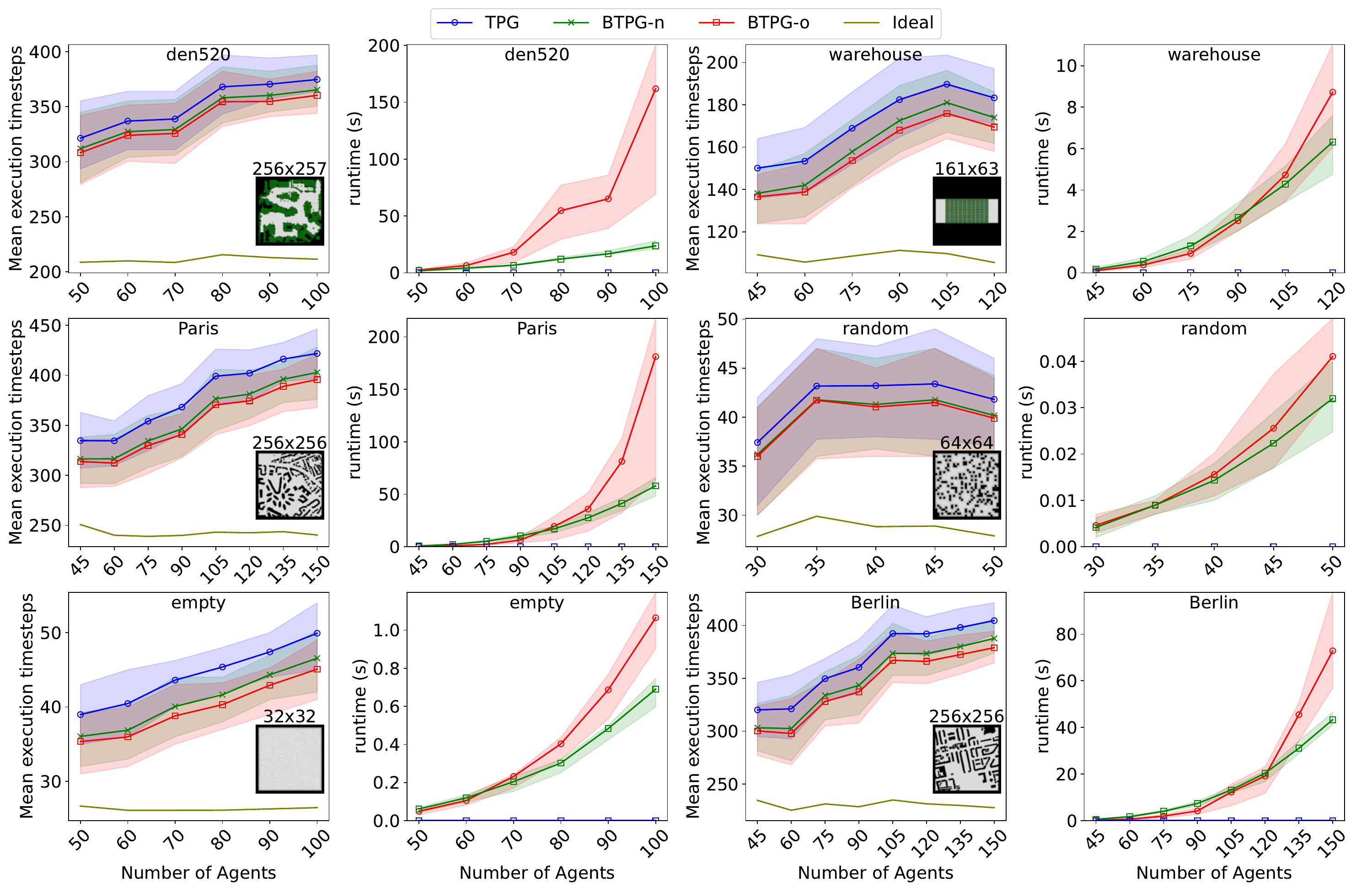}
     \label{random}
    \caption{Performance of TPG, BTPG-naïve, and BTPG-optimized. The shaded area represents the interquartile range (25\%-75\%) of the data distribution. The runtime of BTPG-n and BTPG-o is the time for the two algorithms to finish finding their respective bidirectional pairs. Even though BTPG-o takes longer to finish finding bidirectional pairs, it finds around 2x more bidirectional pairs than BTPG-n, so it is actually more efficient than BTPG-n (see Figure \ref{fig:anytime}). }
    \label{fig:comparison}
\end{figure*}

\begin{lemma}
\label{the: prev}
When agents face a deadlock caused by a cycle in a BTPG, no agents have visited any states in the cycle. 
\end{lemma}
 \begin{proof}
 Suppose one state $v_i^m$ in the cycle is visited, then the edge $(v_i^m, u_j^n)$ in the cycle that points from this state is satisfied. This indicates that the agent $j$ can move to state $u_j^n$. We can similarly propagate the logic and show all states in the cycles can be visited, which is a contradiction to a deadlock. Therefore, all states in the cycle have not been visited. 
 \end{proof}
\begin{lemma}
\label{the: btype2}
    If edge $(v_i^n, v_j^m) \in E_{pair}$ is selected during execution, then agent $n$ has already visited state $v_{i-1}^n$.
\end{lemma}
\begin{proof}
This can be proved directly from the BTPG execution policy.
\end{proof}
\begin{corollary}
\label{cor: btype at the position}
When agents face a deadlock caused by a cycle involving edge $(v_i^n, v_j^m) \in E_{pair}$, agent $n$ is at state $v_{i-1}^n$.
\end{corollary}
\begin{proof}
    From \Cref{the: prev}, agent $n$ must be before state $v_{i}^n$. From \Cref{the: btype2}, agent $n$ is at or after state $v_{i-1}^n$. Thus, agent $n$ must be at state $v_{i-1}^n$.
\end{proof}

\begin{theorem}
\label{the: third ignore}
    If a cycle contains a vertex $v_i^n$ and an edge in $E_{pair}$ that points from $v_j^n, j>i$, then this cycle will not lead to a deadlock.
\end{theorem}
\begin{proof}
 Suppose that this cycle will lead to a deadlock. When the deadlock occurs, by \Cref{cor: btype at the position}, agent $n$ is at state $v_{j-1}^n$. By \Cref{the: prev}, agent $n$ has not visited state $v_{i}^n$. Thus, state $v_{i}^n$ should be after state $v_{j-1}^n$, contradicting the assumption of $j>i$. Thus, the theorem holds.
\end{proof}

We denote the cycles described in \Cref{the: third ignore} as \emph{non-deadlock cycles}. 
The cycle in Figure \ref{fig: example BTPG} is a non-deadlock cycle as it contains both vertex $A^1_2$ and edge $(F^1_4,E^3_3)$. 

\begin{property}[Valid BTPG-optimized]\label{pro:BTPG-o}
A BTPG is valid if it does not contain any cycles apart from rotation cycles, self cycles, and non-deadlock cycles.    
\end{property}

Therefore, our second algorithm BTPG-o extends BTPG-n by considering non-deadlock cycles. Specifically, in \Cref{algo: hasCycle}, we use $V_{vis}$ to record all vertices that have been visited along the current DFS branch and skip edge $(v_i^m, v_j^n) \in E_{pair}$ if the current DFS branch contains a state $v_{i'}^m, i'<i$ (\Cref{line:skip-e}). Furthermore, in \Cref{algo:Constructing BTPG - baseline}, we repeat the examining process of type-2 edges until no new bidirectional pair is discovered (\Cref{line:Epair-updated}) for the following reason.

Let us consider the example in Figure \ref{fig: example BTPG}.
If we examine $(D^2_4,C^3_2)$ before $(F^1_4,E^3_3)$, then $(D^2_4,C^3_2)$ is not added into $E_{pair}$ because, based on the available information at that point, cycle $E^3_3 \rightarrow C^2_3 \rightarrow A^1_2  \rightarrow E^1_3 \rightarrow F^1_4 \rightarrow E^3_3$ is considered a cycle that leads to a deadlock. However, after $(F^1_4, E^3_3)$ is examined and added into $E_{pair}$, this cycle becomes a non-deadlock cycle, and thus, $(D^2_4,C^3_2)$ and $(E^3_3,C^2_3)$ are then added into $E_{pair}$.
Thus, to maximize the size of $E_{pair}$, we repeat the examining process until no new bidirectional pair is discovered. This is unique to BTPG-o, where the validity condition becomes ``easier" to satisfy as more bidirectional pairs are added since our DFS skips some edges in bidirectional pairs (see \Cref{line:skip-e} in Algorithm 2). 

\section{Empirical Evaluation}

\begin{table*}[th!]
    \centering
    \resizebox{\textwidth}{!}{
    \begin{tabular}{|c|c|c|c|c|c|c|c|c|c|c|c|c|}
    
    \hline
    &\multicolumn{2}{|c|}{\texttt{den520}}& 
    \multicolumn{2}{|c|}{\texttt{warehouse}} &
    \multicolumn{2}{|c|}{\texttt{Paris}} &
    \multicolumn{2}{|c|}{\texttt{random}} &
    \multicolumn{2}{|c|}{\texttt{empty}} &
    \multicolumn{2}{|c|}{\texttt{Berlin}} \\ \hline
    & BTPG-n& BTPG-o& BTPG-n& BTPG-o& BTPG-n&BTPG-o& BTPG-n&BTPG-o& BTPG-n&BTPG-o& BTPG-n&BTPG-o \\ \hline
    Mean imp.& 5.9\%& \textbf{8.9}\%& 12.1\%& \textbf{17.9}\%& 10.6\%& \textbf{14.6}\%& 12.9\%& \textbf{15.2}\%& 14.5\%& \textbf{20.9}\%& 9.6\%& \textbf{14.6}\% \\ \hline
    Median imp.& 5.7\%& \textbf{8.1\%}& 12.3\%& \textbf{17.8\%}& 10.5\%& \textbf{14.2\%}& 10.3\%& \textbf{12.2\%}& 14.0\%& \textbf{20.0\%}& 9.2\%& \textbf{14.2\%} \\ \hline
    Max imp.& 14.0\%& \textbf{19.3\%}& 25.0\%& \textbf{35.3\%}& 20.7\%& \textbf{25.4\%}& 50.0\%& \textbf{63.6\%}& 37.5\%& \textbf{42.9\%}& 21.3\%& \textbf{24.7\%} \\ \hline
    Min imp.& 1.1\%& \textbf{2.9}\%& 2.7\%& \textbf{6.7}\%& 4.1\%& \textbf{7.0}\%& 0.0\%& 0.0\%& 3.3\%& \textbf{7.4}\%& 3.2\%& \textbf{7.6}\% \\ \hline \hline
     \# Type-2 edges& \multicolumn{2}{|c|}{26,738} & \multicolumn{2}{|c|}{15,084} & \multicolumn{2}{|c|}{24,401} & \multicolumn{2}{|c|}{1,151} & \multicolumn{2}{|c|}{2,898} & \multicolumn{2}{|c|}{27,922} \\ \hline
    \# Singleton edges & \multicolumn{2}{|c|}{2,172} & \multicolumn{2}{|c|}{1,051} & \multicolumn{2}{|c|}{2,688} & \multicolumn{2}{|c|}{146} & \multicolumn{2}{|c|}{817} & \multicolumn{2}{|c|}{2,265} \\ \hline 
    \# Bi-Pairs found& 663& \textbf{1,008}& 368& \textbf{532}& 1,254& \textbf{1,725}& 60& \textbf{76}& 248& \textbf{360}& 1018& \textbf{1,409} \\ \hline
    \# Used Bi-Pairs & 38& \textbf{68}& 36& \textbf{57}& 82& \textbf{125}& 6& 6& 25& \textbf{40}& 68& \textbf{112}\\ \hline \hline
    BTPG runtime (s)& 23.72& 161.88& 6.31& 8.72& 58.12& 181.33& 0.03& 0.04& 0.69& 1.06& 43.32& 72.87 \\ \hline
    MAPF runtime (s)& \multicolumn{2}{|c|}{1.93} & \multicolumn{2}{|c|}{9.44} & \multicolumn{2}{|c|}{1.12} & \multicolumn{2}{|c|}{19.80} & \multicolumn{2}{|c|}{1.11} & \multicolumn{2}{|c|}{9.90} \\ \hline

    \end{tabular}}
    
   \caption{Statistics of BTPG-naïve/optimized. The number of agents for the six maps selected for the statistics are 100, 120, 150, 50, 100, and 150, respectively. All data in the second block are averages of 10 scenarios for each map; Bi-Pairs: bidirectional pairs; imp.: improvement; Singleton edges: type-2 edges that cannot be grouped. 
   }
   \label{tab:Statistics of BTPG-naive}
 \end{table*}

Our experiments work as follows. First, we use the optimal MAPF solver CBSH2-RTC \cite{li_pairwise_2021} to obtain the MAPF plan for each given MAPF instance. We have in total 3,900 simulations, consisting of six different benchmark maps from~\cite{stern_multi-agent_2021}, as shown in \Cref{fig:comparison}, each map with five to eight different numbers of agents, and each number of agents with ten random instances. The largest number of agents for each map is determined by the largest number of agents CBSH2-RTC can solve within 2 minutes. 
Then, we convert each MAPF plan into a TPG (for which the runtime is negligible). 
Next, we construct BTPG using our proposed algorithms BTPG-naive and BTPG-optimized. Last, we simulate the TPG and BTPG execution policies with 10\% agents having a 30\% chance of being delayed by 5 timesteps at each non-delayed timestep. Each instance is simulated with 10 random seeds, resulting in 3,900 simulations per algorithm.
All algorithms were implemented in C++ \footnote{https://github.com/YifanSu1301/BTPG}, and all experiments were run on a PC with a 3.60 GHz Intel i7-12700KF CPU and 32 GB of RAM. 
 

\paragraph{Execution Time} \Cref{fig:comparison} reports the \emph{mean execution timesteps}, namely the average number of timesteps that each agent takes to reach its target location during execution. We include a lower bound called \emph{Ideal}, which is the sum of timesteps in the original MAPF plan plus the total delay timesteps of the delayed agents, then divided by the number of agents. Essentially, \emph{Ideal} represents an overly optimistic scenario where delayed agents do not cause any additional waits for other agents. As shown, BTPG-o consistently outperforms BTPG-n, which in turn consistently outperforms TPG, across all maps and all numbers of agents. 

To quantify the execution time improvement of BTPG over TPG, we define \emph{improvement} as $\frac{T_{TPG} - T_{BTPG}}{T_{TPG}-T_{Ideal}}$, where $T_{TPG}$, $T_{BTPG}$, and $T_{Ideal}$ are the mean execution timesteps of TPG, BTPG, and Ideal, respectively. \Cref{fig:scatter plot} shows that our improvement is always non-negative over 3,900 simulations, even though, as mentioned in \Cref{sec:BTPG}, we could have negative improvements in theory. The top block of \Cref{tab:Statistics of BTPG-naive} further provides detailed numbers on the improvements over the instances with the largest number of agents for each map. The median improvement of BTPG-o is in the range of 8-20\%, with maximum values around 19-64\%.

\begin{figure}[t!]
     \centering     
     \includegraphics[width=0.39\textwidth]{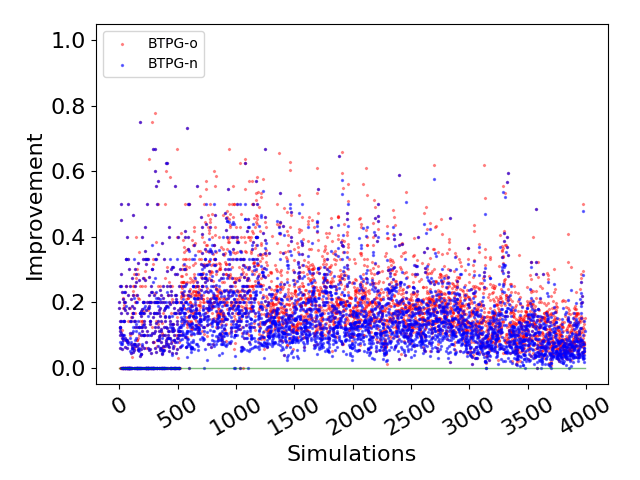}
     \caption{Improvement of BTPG over TPG per instance. We find no instances with negative improvements and 3.6\%, 23.3\%,  41.9\%, and 31.2\% of instances with no improvement, 0-10\% improvement, 10-20\% improvement, and $>$20\% improvement across all simulations, respectively.}
     \label{fig:scatter plot}
 \end{figure}
 
\paragraph{Bidirectional Pairs} The middle block of \Cref{tab:Statistics of BTPG-naive} reports how many bidirectional pairs that our algorithms find and that are useful. While the TPG has thousands or even ten thousands of type-2 edges, only about 10\% of these edges are not grouped and have the chance to be changed to bidirectional pairs. After BTPG-o evaluates these singleton type-2 edges, about 50\% are changed to bidirectional pairs. We define a bidirectional pair as \emph{used} if the agents select the reversed edge rather than the original TPG type-2 edge during execution. We see that roughly 10\% of bidirectional pairs are used. Although this is a small percentage over the total number of type-2 edges, these used bidirectional pairs are the contributors to the significant reduction in execution time that we reported above.



\begin{figure}[t!]
    \centering
    \includegraphics[width=0.45\textwidth]{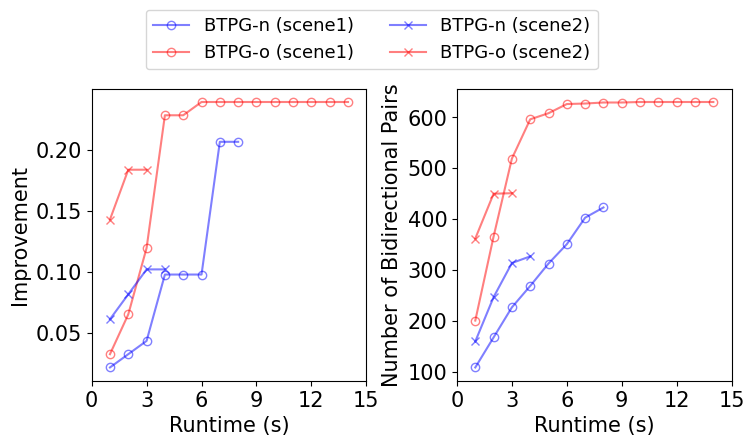}
    \caption{Anytime behavior of BTPG-n/o on \texttt{warehouse} with 120 agents. Scenes 1 and 2 are scenarios with the longest and shortest BTPG-o runtimes, respectively. The appendix has similar results on other maps.}
    \label{fig:anytime}
\end{figure}

\paragraph{Runtime}
Both \Cref{fig:comparison} and the bottom block of \Cref{tab:Statistics of BTPG-naive} report the CPU runtime of our algorithms. As expected, BTPG-o is slower than BTPG-n. 
However, we find that the longer runtime of BTPG-o is due to finding more bidirectional pairs rather than inefficiency. \Cref{fig:anytime} plots the anytime behavior of both algorithms. For any given cut-off time, BTPG-o finds more bidirectional pairs and thus leads to better improvement than BTPG-n. BTPG-o is more efficient because its DFS can skip over more edges than BTPG-n.



\section{Conclusion}
We constructed a new graphical representation of passing orders in the MAPF plan, BTPG, by proposing the concept of bidirectional pairs. The main difference between BTPG and TPG lies in the fact that agents can switch the order of passing certain locations during execution. Two algorithms, BTPG-n and BTPG-o, are proposed to construct a BTPG. The results indicate that following BTPGs consistently outperform following TPGs by 8-20\% when agents get delayed. Also, we show our proposed algorithms are anytime, and given a fixed time budget, BTPG-o outperforms BTPG-n. Overall, we convincingly show that allowing switching dependencies in a MAPF plan allows us to improve execution time without replanning. 

\newpage
\section*{Acknowledgments}
The research is supported by the National Science Foundation (NSF) under grant number 2328671. The views and conclusions contained in this document are those of the authors and should not be interpreted as representing the official policies, either expressed or implied, of the sponsoring organizations, agencies, or the U.S. government.

\bibliography{BTPG}
\newpage
\appendix

\setcounter{figure}{0}
\renewcommand{\thefigure}{A\arabic{figure}}
\setcounter{table}{0}
\renewcommand{\thetable}{A\arabic{table}}

\section{More Results of the Anytime Behavior}
In the paper, we plot the anytime behavior of BTPG-n and BTPG-o on map \texttt{warehouse}. Here, we provide the results on other maps. Again, Scenes 1 and 2 are the instances with the longest and shortest runtime of BTPG-o in the experiments for that map, respectively.

\begin{figure}[h]
    \centering
    \includegraphics[width=0.45\textwidth]{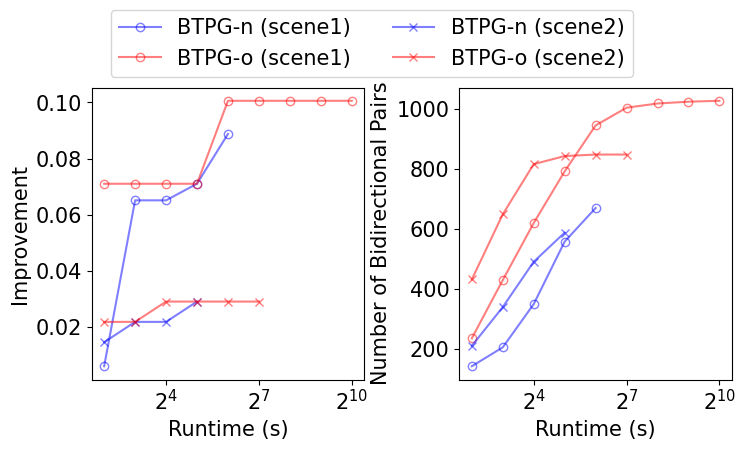}
    \caption{\texttt{den520}}
\end{figure}
\begin{figure}[h]
    \centering
    \includegraphics[width=0.45\textwidth]{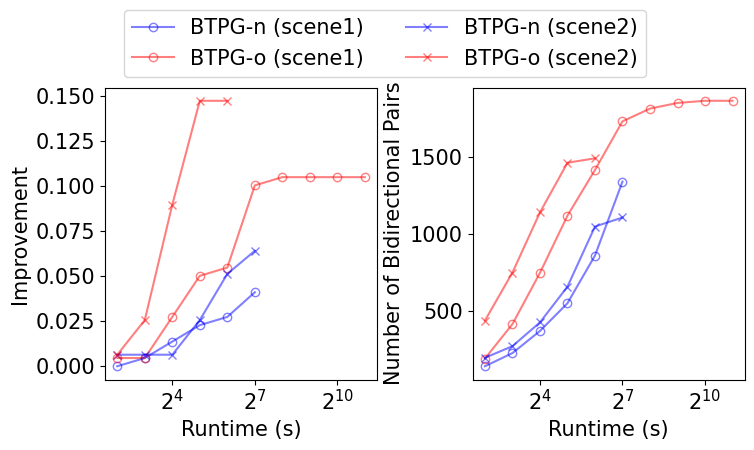}
    \caption{\texttt{Paris}}
\end{figure}
\begin{figure}[h]
    \centering
    \includegraphics[width=0.45\textwidth]{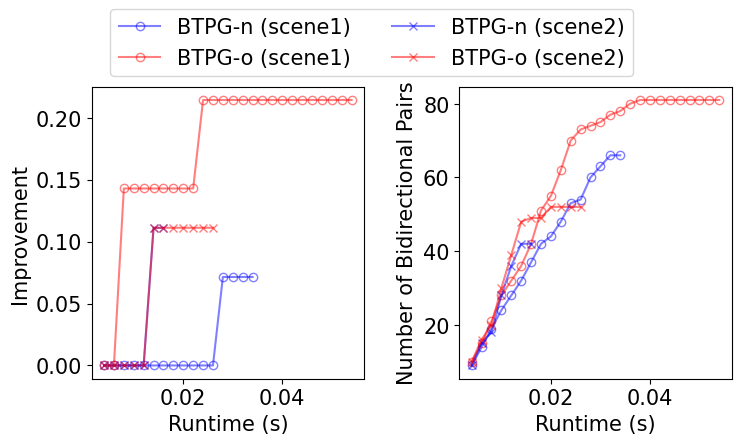}
    \caption{\texttt{random}}
\end{figure}
\begin{figure}[h]
    \centering
    \includegraphics[width=0.45\textwidth]{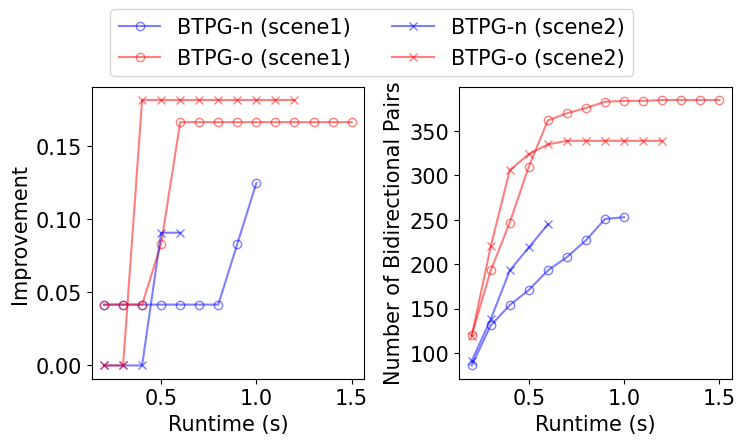}
    \caption{\texttt{empty}}
\end{figure}
\begin{figure}[h]
    \centering
    \includegraphics[width=0.45\textwidth]{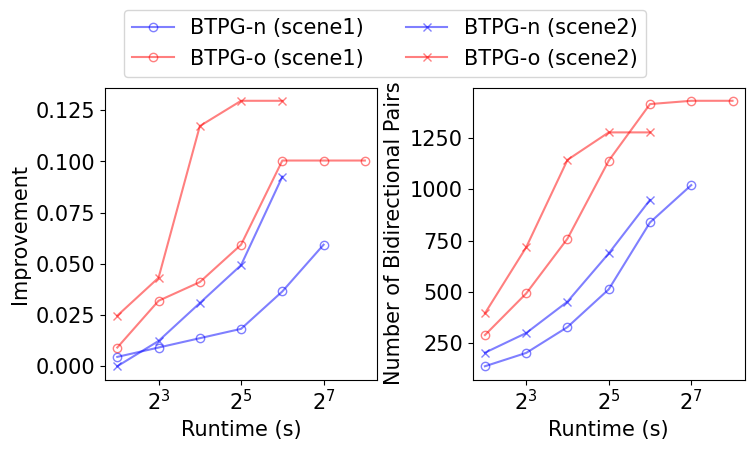}
    \caption{\texttt{Berlin}}
\end{figure}
\end{document}